\documentclass[twoside,11pt]{article}
\usepackage{jmlr2e}

\usepackage{amsmath}
\usepackage{amssymb}
\usepackage{amsfonts}
\usepackage{dsfont}
\usepackage{array}
\usepackage{graphicx}
\usepackage{subfigure}
\usepackage{epsfig}
\usepackage{color}
\usepackage{cite}      
\usepackage{algorithm,algorithmic}
\usepackage{subfigure} 
\usepackage{url}       
\usepackage{epsfig}

\newcommand{\bP}{\mathbb{P}}

\newcommand{\defas}{\triangleq}
\newcommand{\ci}{\perp\!\!\!\perp}
\newtheorem{coro}{Corollary}
\newtheorem{prop}{Proposition}
\newtheorem{defn}{Definition}
\newtheorem{assumption}{Assumption}
\begin{document}
\title{Greedy Learning of Markov Network Structure
\thanks{The results in this paper were presented in \citep{Netetal10} without proofs of the theorems.
This paper includes all the proofs along with simulations.}
}



\author{\name Praneeth Netrapalli \email praneethn@utexas.edu \\
       \name Siddhartha Banerjee \email sbanerjee@mail.utexas.edu \\
       \name Sujay Sanghavi \email sanghavi@mail.utexas.edu \\
       \name Sanjay Shakkottai \email shakkott@mail.utexas.edu \\
       \addr 
       The University of Texas at Austin\\
       Austin, TX 78712, USA}
\editor{}


\maketitle

\begin{abstract}
We propose a new yet natural algorithm for learning the graph structure of general discrete graphical models (a.k.a. Markov random fields) from samples. Our algorithm finds the neighborhood of a node by sequentially adding nodes that produce the largest reduction in empirical conditional entropy; it is greedy in the sense that the choice of addition is based only on the reduction achieved at that iteration. Its sequential nature gives it a lower computational complexity as compared to other existing comparison-based techniques, all of which involve exhaustive searches over every node set of a certain size. Our main result characterizes the sample complexity of this procedure, as a function of node degrees, graph size and girth in factor-graph representation. We subsequently specialize this result to the case of Ising models, where we provide a simple transparent characterization of sample complexity as a function of model and graph parameters.

For tree graphs, our algorithm is the same as the classical Chow-Liu algorithm, and in that sense can be considered the extension of the same to graphs with cycles.

\end{abstract}

\section{Introduction}
Markov Random Fields (MRF), or undirected graphical models, encode conditional independence relations between random variables. Depending on the application at hand, nodes of a graphical model may represent people, genes, languages, processes, etc., while the
graphical model illustrates certain conditional dependencies among them (for example, influence in a
social network, physiological functionality in genetic networks, etc.). Often the knowledge of the
underlying graph is not available beforehand, but must be inferred from certain observations of the
system. In mathematical terms, these observations correspond to samples drawn from the underlying distribution.
Thus, the core task of structure learning is that of inferring conditional dependencies between
random variables from i.i.d samples drawn from their joint distribution. The importance of the MRF in understanding the underlying system makes structure learning an important primitive for studying such systems. 

This paper proposes a new yet natural method to infer the graph structure of an MRF from samples, and analytically characterizes its sample complexity in terms of graph and model parameters. Our algorithm is based on the fact that the graph neighborhood of a node is also its Markov blanket, and conditioned on it the node's variable is independent of all others. We build this neighborhood in a greedy fashion, by sequentially adding the nodes that give the biggest reductions in conditional entropy. Our analytical results -- both for general models and Ising models -- require lower bounds on the girth of the graph. In practice -- for both synthetic examples and a real dataset drawn from senate voting records -- our algorithm is seen to perform quite well even for graphs with lots of small cycles. 

Our algorithm has lower computational complexity as compared to other algorithms that are not tailored to specific model classes (note that if we know a-priori that we are looking for an Ising model, or a Gaussian one, faster methods exist). We review and compare our algorithms to existing literature below. We also elaborate on the sense in which our algorithm can be thought of as an extension of the Chow-Liu algorithm \citep{chowliu} to graphs with cycles.


The remaining sections are organized as follows. In Section \ref{section_prelim}, we review graphical models and some results from information theory, and set up the structure learning problem. Our new structure learning algorithm, GreedyAlgorithm$(\epsilon)$, is given in Section \ref{section_algo}. Next, in Section \ref{section_suffcond}, we develop a sufficient condition for the correctness of the algorithm for general graphs. To demonstrate the applicability of this condition, we translate it into equivalent conditions for learning an Ising model in Section \ref{section_isingguarantees}.
We present simulation results evaluating our algorithm in Section \ref{section_simulations}.
We discuss future work and conclude in Section \ref{section_discussion}. The proofs of theorems are in the Appendix.

\subsection{Related Work}

Learning the structure of graphical models is a well-established problem; existing work falls into two broad categories. The first category involves methods tailored for a specific parametric form of the probability distribution. In particular, when a parametric family is known, the (log) likelihood of the data is written as a function (often convex) of the parameters of the distribution; this likelihood is then maximized, often with added regularizers like an $\ell_1$ penalty, to find the parameters and hence the graph structure. Examples in this category include \citep{RavWaiRasYu11,ElK08,FurBen07,ZhoLafWas10,AnaTan11:Gau} for Gaussian graphical models, \citep{ravikumar,banerjee,santhanam} for Ising models, \citep{JalRavVasSan11} for general discrete {\em pairwise} graphical models.

The other category of graphical model structure learning algorithms are those that do not need to assume (and cannot leverage) specific parametric forms of the distribution. Rather, they are based on the notion that a node's Markov blanket, i.e. its neighborhood in the graphical model, makes a node conditionally independent of other nodes. Examples of such algorithms include \citep{chowliu,koller,bresler,AnaTan11:Isi,BenMon}. All of these methods involve an exhaustive search over all subsets of nodes upto a certain size $d$ -- typically the degree of the node whose neighborhood we are trying to find. This results in a high computational complexity for the algorithms. 

This paper falls into the latter category, but avoids the high computational complexity of exhaustive searching by building the sets in a greedy fashion instead. 

Finally, we note that if the graph is a tree, our method is equivalent to the classical Chow-Liu method \citep{chowliu}. In particular, \citep{chowliu} involves making a max-weight spanning tree where the edge weights are mutual information. From any given fixed node's perspective, this algorithm adds edges in the same order as our algorithm; i.e. greedily adding nodes that give the biggest reduction in conditional entropy. In that sense, our algorithm can be considered a generalization of \citep{chowliu} to graphs with cycles.

\section{Preliminaries}
\label{section_prelim}

We now setup the (standard) graphical model structure learning problem. Let $X$ be a $p$-dimensional random vector $\{X_1,X_2,\ldots,X_p\}$, where each component $X_i$ of $X$ takes values in a finite set $\mathcal{X}$. We use the shorthand notation $P(x_i) \defas \mathbb{P}(X_i=x_i), x_i\in\mathcal{X}$, and similarly for a set $A\subseteq \{1,2,\ldots,p\}$, we define $P(x_A)\defas\mathbb{P}(X_A=x_A), x_a\in\mathcal{X}^{|A|}$, where $X_A\defas\{X_i|i\in A\}$.

Let $G$ be the Markov graph of $X$, with vertex set $V$ (one node $i\in V$ for each variable $X_i$), and edge set $E$. In particular, this means that the probability distribution of $X$ satisfies the {\em local markov property} \citep{Lau96} with respect to $G$: for every $i\in V$, if its neighborhood in $G$ is $N(i)$, then for any set $B\in V\setminus \{i\}\cup N(i)$, we have that $P(x_i|x_{N(i)},x_B)=P(x_i|x_{N(i)})$ for all $(x_i,x_{N(i)},x_B)$. 

Our goal is to learn the structure of $G$ -- i.e. the set of its edges $E$ -- from $n$ vector samples $x^{(1)},\ldots,x^{(n)}$, which are drawn iid from the joint distribution. The {\em empirical distribution} $\widehat{P}$ is defined as follows; for any set $A$ of variables (nodes) and corresponding values $x_A$,
\begin{equation*}
\widehat{P}(x_A) ~ \defas ~ \frac{1}{n}\sum_{i=1}^n\mathds{1}_{\{x_A^{(i)}=x_A\}}
\end{equation*}
The empirical entropy and conditional entropy refer to the corresponding quantities for this empirical distribution $\widehat{P}$. In this paper we will refer to the true entropies by $H$ and the empirical ones by $\widehat{H}$. The following fact is immediate from conditional independence and the Data Processing Inequality, see \citep{coverthomas}.
\begin{prop}
\label{prop_hmin}
For any node $i\in V$, its neighborhood $N(i)$ in $G$, and any set $A\subseteq V\setminus\{i\}$, we have that
\begin{equation*}
H(X_i|X_{N(i)}) \leq H(X_i | X_{A}),
\end{equation*}
\end{prop}
Motivated by this relationship, \citep{koller} advocated finding $N(i)$ by exhaustive searching over all sets of size less than $d$, where $d$ is the (upper bound on the) degree of node $i$. Our method avoids this exhaustive search, but builds the neighborhood in a sequential greedy fashion. 

Of course, any algorithm would need to work with samples, which in our case would be empirical entropies. We find the following result -- obtained by combining Theorem $16.3.2$ and Lemma $16.3.1$ from \citep{coverthomas} -- useful in translating between conditions on the true and empirical entropy quantities.

\begin{prop}
\label{prop_hnearp}
Let $P$ and $Q$ be two probability mass functions in a finite set $\mathcal{X}$, with entropies $H(P)$ and $H(Q)$ respectively, and with total variational distance $||P-Q||_1$ given by:
\begin{equation*}
||P-Q||_1=\sum_{x\in \mathcal{X}}|P(x)-Q(x)|.
\end{equation*}
Then 
\begin{equation}
\label{eqn_hleqp}
|H(P)-H(Q)|\leq -||P-Q||_1\log\frac{||P-Q||_1}{|\mathcal{X}|}.
\end{equation}
Further, if the relative entropy between them is given by $D(P||Q)$, then
\begin{equation}
\label{eqn_pleqh}
D(P||Q)\geq \frac{1}{2\log 2}||P-Q||_1^2.
\end{equation}
\end{prop}

We characterize the sample complexity of our algorithm for a class of graphs and models. We specify these models in terms of their factor graph, which we define below for completeness.
\begin{defn}
  \textbf{(Factor Graph)} Given a graphical model $G(V,E)$ its factor graph is a bipartite graph $G_f$ with vertex set $V\cup C$
where each vertex $c \in C$ corresponds to a maximal clique in $G$. For any $v\in V$ and $c\in C$, there is an edge $\{v,c\}$ in
$G_f$ if and only if $v \in c$ in $G$.
\end{defn}
We have the following simple lemma relating the distance between two nodes $i,j\in V$ in the graphs $G$ and $G_f$.
\begin{lemma}\label{lemma:factorgraph-actgraph-distance}
  Given a graph $G$, let $G_f$ be its factor graph. Then for every $i,j\in V$ we have $d_f(i,j) = 2d(i,j)$ where
$d$ and $d_f$ are the distances between $i$ and $j$ in $G$ and $G_f$ respectively.
\end{lemma}

\section{The GreedyAlgorithm$(\epsilon)$ Structure Learning Algorithm}
\label{section_algo}

We now present our method, GreedyAlgorithm$(\epsilon)$, which proceeds by finding the Markov neighborhood of each node separately. For the neighborhood of node $i$, it starts with an empty set and iteratively adds nodes that bring the largest additional decrease in (empirical) conditional entropy. It stops when this decrease is less than $\epsilon\slash 2$. The formal specification is presented in Algorithm \ref{GA}.


\begin{algorithm}[h]
\caption{GreedyAlgorithm($\epsilon$)}
\label{GA}
\begin{algorithmic}[1]
 \FOR{$i \in V$}
    \STATE complete $\leftarrow$ FALSE
    \STATE $\widehat{N}(i) \leftarrow \Phi$
    \WHILE{!complete}
      \STATE $j = \underset{k \in V \setminus \widehat{N}(i)}{\operatorname{argmin}} \widehat{H}(X_i \mid X_{\widehat{N}(i)}, X_k)$
      \IF{$\widehat{H}(X_i \mid X_{\widehat{N}(i)}, X_j) < \widehat{H}( X_i \mid X_{\widehat{N}(i)}) - \frac{\epsilon}{2}$}
	  \STATE $\widehat{N}(i) \leftarrow \widehat{N}(i) \cup \{j\}$
      \ELSE
	  \STATE complete $\leftarrow$ TRUE
      \ENDIF
    \ENDWHILE
 \ENDFOR
\end{algorithmic}
\end{algorithm}
%
 
\section{Sufficient Conditions for General Discrete Graphical Models}
\label{section_suffcond}

In this section, we provide guarantees for general discrete graphical models, under which GreedyAlgorithm$(\epsilon)$ recovers the graphical model structure exactly. First, using an example, we build up intuition for the sufficient conditions, and define two key notions: non-degeneracy conditions and correlation decay. Our main result is presented in Section \ref{subsection_guarantees}, wherein we give a sufficient condition for the correctness of the algorithm in general discrete graphical models.  

\subsection{Non-Degeneracy and Correlation Decay}
\label{subsec_assump}

Before analyzing the correctness of structure learning from samples, a simpler problem worth considering is one of algorithm consistency, i.e., does the algorithm succeed to identify the true graph \textit{given the true conditional distributions} (or in other words, given an infinite number of samples). It turns out that the algorithm as presented in Algorithm \ref{GA} does not even possess this property, as is illustrated by the following counter-example

Let $V = \{0,1,\cdots, D,D+1\}$, $X_i \in \{-1,1\} \forall i \in V$ and $E = \left\{ \{0,i\},\{i,D+1\} \mid 1 \leq i \leq D \right\}$.
Let $P(x_V) = \frac{1}{Z} \displaystyle \prod_{\{i,j\} \in E} e^{ \theta x_i x_j}$, where $Z$ is a normalizing constant (this is the classical zero-field Ising model potential). The graph is shown in Fig. \ref{fig_counterexample}.

\begin{figure}[h]
  \centering 
  \includegraphics[width=4cm]{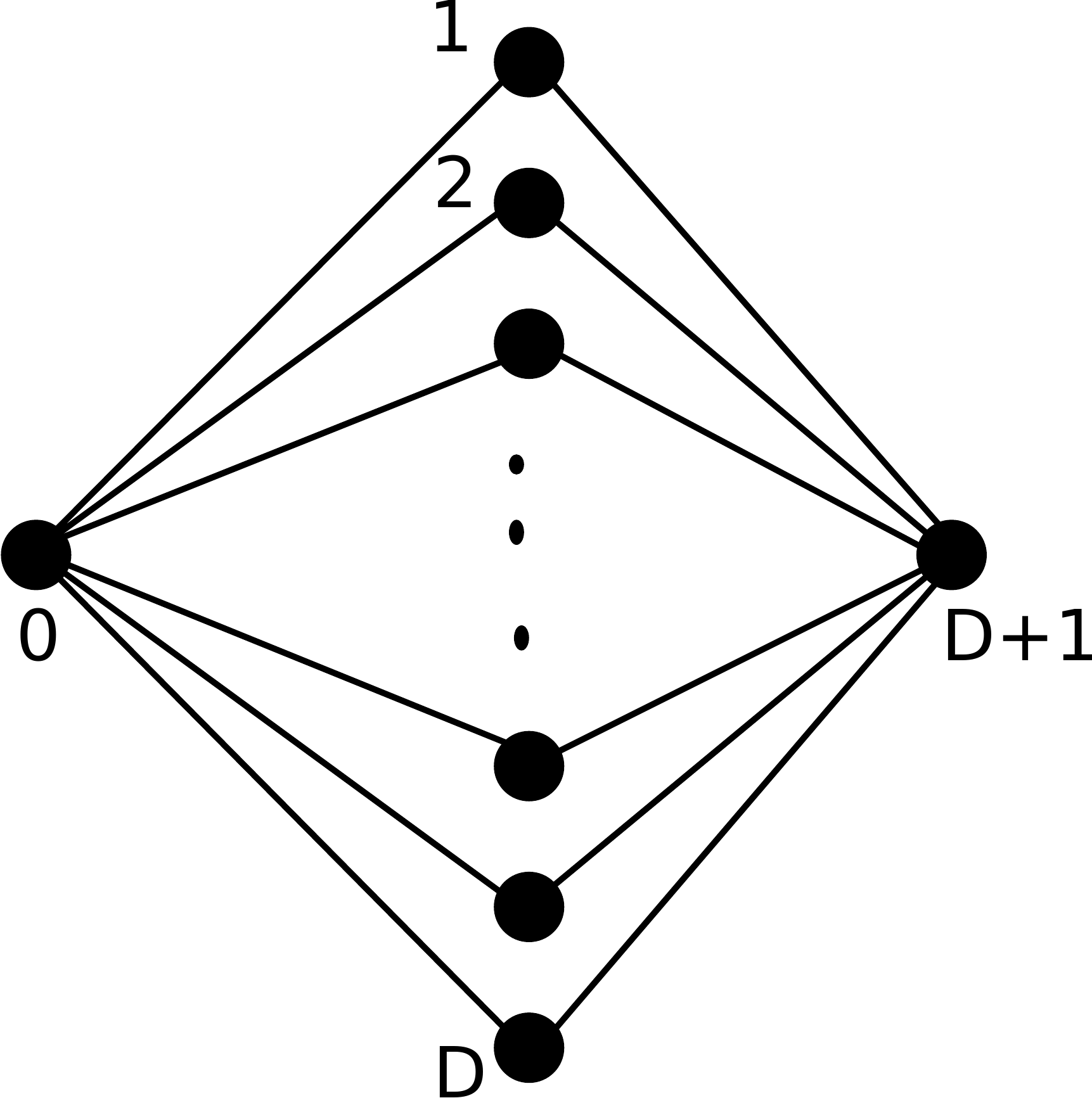} 
  \caption{An example of adding spurious nodes: Execution of GreedyAlgorithm$(\epsilon)$ for node $0$ adds node $D+1$ in the first iteration, even though it is not a neighbor.}
\label{fig_counterexample}
\end{figure}

Suppose the actual entropies are given as input to Algorithm \ref{GA}. It can be shown in this case that for a given $\theta$, there exists a $D_{\mbox{thresh}}$ such that if $D>D_{\mbox{thresh}}$, then the output of Algorithm \ref{GA} will select the edge $\{0,D+1\}$ in the first iteration. This is easily understood because if $D$ is large, the distribution of node $0$ is best accounted for by node $D+1$, although it is not a neighbor. Thus, even with exact entropies, the algorithm will always include edge $(0,D+1)$, although it does not exist in the graph. 

The algorithm can however easily be shown to satisfy the following weaker consistency guarantee: given infinite samples, for any node in the graph, the algorithm will return a \textit{super-neighborhood}, i.e., a superset of the neighborhood of $i$. This suggests a simple fix to obtain a consistent algorithm, as we can follow the greedy phase by a `node-pruning' phase, wherein we test each node in the neighborhood of a node $i$ returned by the algorithm (to do this, we can compare the entropy of $i$ conditioned on the neighborhood with and without a node, and remove it if they are the same). However the problem is complicated by the presence of samples, as pruning a large super-neighborhood requires calculating estimates of entropy conditioned on a large number of nodes, and hence this drives up the sample complexity. In the rest of the paper, we avoid this problem by ignoring the pruning step, and instead prove a stronger correctness guarantee: given any node $i$, the algorithm always picks a \textit{correct} neighbor of $i$ as long as any one remains undiscovered. Towards this end, we first define two conditions which we require for the correctness of GreedyAlgorithm$(\epsilon)$.

\noindent \begin{assumption}[\textbf{Non-degeneracy}] \label{assump_nondegeneracy} Choose a node $i$. Let $N(i)$ be the set of its neighbors.
Then $\exists \epsilon > 0$ such that $\forall \; A \subset N(i)$, $\forall \; j \in N(i) \setminus A$ and
$\forall \; l \in N(j) \setminus \{i\}$, we have that 

\vspace{-0.5cm}
\begin{equation}\label{eqn_nondegen_2}
H(X_i\mid X_A) - H(X_i\mid X_A, X_j) > \epsilon \mbox{ and}
\end{equation}
\vspace{-0.5cm}
\begin{equation}\label{eqn_nondegen_1}
H(X_i\mid X_A, X_l) - H(X_i\mid X_A, X_j, X_l) > \epsilon 
\end{equation}

\end{assumption}

Assumption \ref{assump_nondegeneracy} is illustrated in Fig. \ref{fig_assump_nondegen}.
\begin{figure}[h]
  \centering 
  \includegraphics[width=4cm]{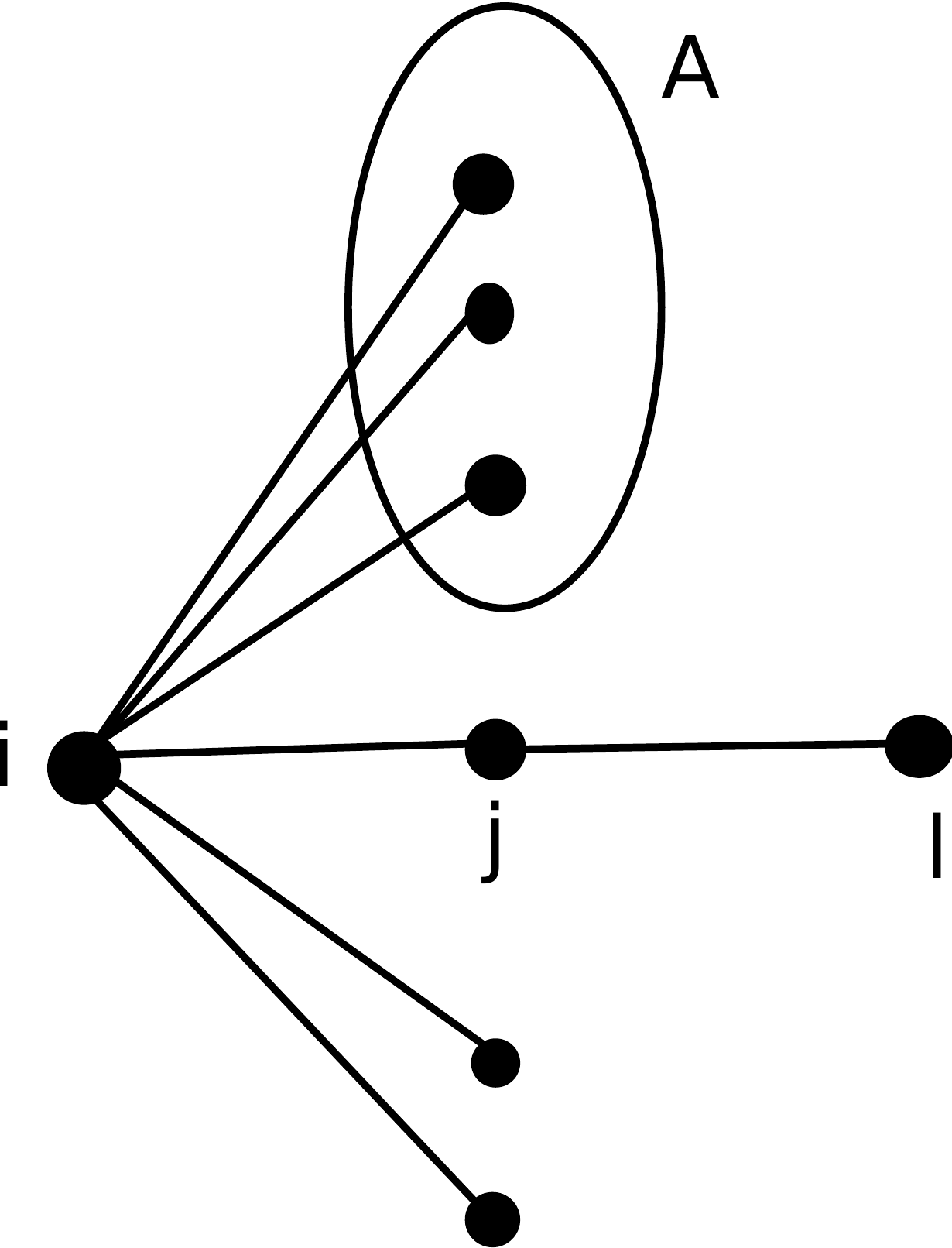}
  \caption{Non-degeneracy condition for node $i$: $(i)$ Entropy of $i$ conditioned on any sub-neighborhood $A$ reduces by at-least $\epsilon$ if any other neighbor $j$ is added to the conditioning set, $(ii)$ Entropy of $i$ conditioned on $A$ and a two hop neighbor $l$ reduces by at-least $\epsilon$ if the corresponding one hop neighbor $j$ is added to the conditioning set}
  \label{fig_assump_nondegen}
\end{figure}

\noindent \begin{assumption} [\textbf{Correlation Decay}] \label{assump_corrdecay}
Choose a node $i$. Let $N^1(i)$ and $N^2(i)$ be the sets of its $1$-hop and $2$-hop neighbors respectively.
Choose another set of nodes $B$. Let $d(i,B) = \displaystyle \min_{j \in B} d(i,j)$, where
$d(i,j)$ denotes the distance between nodes $i$ and $j$. Then, we have that $\forall x_i,x_{N^1(i)},x_{N^2(i)},x_B$
\begin{align*}
\displaystyle \left|P(x_i,x_{N^1(i)},x_{N^2(i)}\mid x_B) - P(x_i,x_{N^1(i)},x_{N^2(i)})\right| < f(d(i,B))
\end{align*}
where $f$ is a monotonic decreasing function.
\end{assumption}

Assumption \ref{assump_nondegeneracy} (or a variant thereof) is a standard assumption for showing correctness of any structure learning algorithm, as it ensures that there is a \textit{unique} minimal graphical model for the distribution from which the samples are generated. Although the way we state the assumption is tailored to our algorithm, it can be shown to be equivalent to similar assumptions in literature\citep{bresler}. Informally speaking, Assumption \ref{assump_nondegeneracy} states that for node $i$, any $2$-hop neighbor captures less information about node $i$ than the corresponding $1$-hop neighbor. In the case of a Markov Chain, Assumption \ref{assump_nondegeneracy} reduces to a weaker version of an $\epsilon-$Data Processing Inequality (i.e., DPI with an epsilon gap), and in a sense, Assumption \ref{assump_nondegeneracy} can be viewed as a generalized $\epsilon-$DPI for networks with cycles. 

On the other hand, Assumption \ref{assump_corrdecay} along with
large girth implies that the information a node $j$ has about node $i$ is `almost Markov' along the shortest path between $i$ and $j$.
This in conjunction with Assumption \ref{assump_nondegeneracy} implies that for any two nodes $i$ and $k$, the information about $i$
captured by $k$ is less than that captured by $j$ where $j$ is the neighbor of $i$ on the shortest path between $i$ and $k$. It is also known \citep{BenMon} that structure learning is a much harder problem when there is no correlation decay.

\subsection{Guarantees for the Recovery of a General Graphical Model} 
\label{subsection_guarantees}

We now state our main theorem, wherein we give a sufficient condition for correctness of GreedyAlgorithm$(\epsilon)$ in a general graphical model. 

The counter-example given in Section \ref{subsec_assump} suggests that the addition of spurious nodes to the neighborhood of $i$ is related to the existence of non-neighboring nodes of $i$ which somehow accumulate sufficient influence over it. The accumulation of influence is due to slow decay of influence on short paths (corresponding to a high $\theta$ in the example), and the effect of a large number of such paths (corresponding to high $D$). Correlation decay (Assumption \ref{assump_corrdecay}) allows us to control the first. Intuitively, the second can be controlled if the neighborhood of $i$ is `locally tree-like'. To quantify this notion, we define the girth of a graph $\mbox{Girth}(G)$ to be the length of the smallest cycle in the graph $G$. Now we have the following theorem. 

\begin{theorem}\label{thm_main}
Consider a graphical model $G$ where the random variable corresponding to each node takes values in a set $\mathcal{X}$ and satisfies the following:
\begin{itemize}
\item Non-degeneracy (Assumption \ref{assump_nondegeneracy}) with parameter $\epsilon$,
\item Correlation decay (Assumption \ref{assump_corrdecay}) with decay function $f(\cdot)$,
\item Maximum degree $D$.
\end{itemize}
Define $h \defas h(\epsilon,D) \defas \frac{\epsilon^2 |\mathcal{X}|^{-2(D+1)^2}}{64}$ and suppose $f^{-1}(h)$ exists. Further suppose $G_f$
(the factor graph of $G$) obeys the following condition:
\begin{equation}
\label{eqn_gcondn}
\mbox{Girth}(G_f)\defas g_f > 4 \left( f^{-1}\left( h \right) + 1 \right).
\end{equation}
Then, given $\delta > 0$, GreedyAlgorithm$(\epsilon)$ recovers $G$ exactly with probability greater than $1-\delta$ with sample complexity $n = \xi\left(\epsilon^{-4} \log\frac{p}{\delta}\right)$, where $\xi$ is a constant independent of $p,\epsilon$ and $\delta$.

\end{theorem}
The proof follows from the following two lemmas. Lemma \ref{lemma_greedyalgo_requirement} implies that if we had access to actual entropies,
Algorithm \ref{GA} always recovers the neighborhood of a node exactly. Lemma \ref{lemma_greedyalgo_requirement_empirical} shows that
with the number of samples $n$ as stated in Theorem \ref{thm_main}, the empirical entropies are very close to the actual entropies with
high probability and hence Algorithm \ref{GA} recovers the graphical model structure exactly with high probability
even with empirical entropies.

\begin{lemma}\label{lemma_greedyalgo_requirement}
 Consider a graphical model $G$ in which node $i$ satisfies Assumptions \ref{assump_nondegeneracy} and \ref{assump_corrdecay}. Let the girth of $G_f$ be
$ g_f> 4 \left( f^{-1}\left( h \right) + 1 \right)$, where $h$ is as defined in Theorem \ref{thm_main}.
Then, $\forall \; A\subsetneq N(i), \; u \notin N(i), \;\exists \; j \in N(i)\setminus A$ such that
\begin{equation}\label{eqn_greedyalgo_condn}
 H(X_i \mid X_A,X_j) < H(X_i \mid X_A,X_u) - \frac{3\epsilon}{4}
\end{equation}
\end{lemma}
\begin{proof}
If $A$ separates $i$ and $u$ in $G_f$ it also does so in $G$. Then we have that $P(x_i|x_A,x_u) = P(x_i|x_A)$ and hence
$H(X_i \mid X_A,X_u) = H(X_i \mid X_A)$. Then, the statement of the lemma follows from (\ref{eqn_nondegen_2}).

Now suppose $A$ does not separate $i$ and $u$ in $G_f$. Consider the shortest path between $i$ and $u$
in $G_f\setminus A$. Let $j \in N(i) \setminus A$ and $l \in N(j) \setminus \{i\}$
be on that shortest path. Assumption \ref{assump_nondegeneracy} implies that $H(X_i\mid X_A, X_l) -
H(X_i\mid X_A, X_j, X_l) > \epsilon$. Now, choose $B \in V$ such that $A \cup B \cup \{j\}$
separates $i$ and $l$ in $G_f$ and $d_f(i,B) \geq \frac{g_f-4}{2}$, where $g_f$ is the girth of $G_f$.
Note that such a $B$ (possibly empty) exists since the girth of $G_f$ is $g_f$ and if a node in the separator
is a factor node (i.e., not in $V$) then we can replace it by all its neighbors (in $V$).
We then see using Lemma \ref{lemma:factorgraph-actgraph-distance} that $d(i,B) \geq \frac{g_f-4}{4}$.
From Assumption \ref{assump_corrdecay}, we know that
\begin{equation*}
 \hspace{-0.5cm}
 \begin{array}{rl}
  &|P(x_i, x_{N(i)\cup N^2(i)}) - P(x_i, x_{N(i)\cup N^2(i)} \mid x_B)| < f\left(\frac{g_f}{4} - 1 \right)\\
  &\Rightarrow \displaystyle \sum_{x_i, x_A, x_j} |P(x_i, x_A, x_j) - P(x_i, x_A, x_j \mid x_B)| < |\mathcal{X}|^{(D+1)^2} f\left( \frac{g_f}{4} - 1 \right) \;\; \forall \;x_B \\
  & \Rightarrow  H(X_i, X_A, X_j) - H(X_i, X_A, X_j \mid X_B)
	      < - |\mathcal{X}|^{(D+1)^2} f\left( \frac{g_f}{4} - 1 \right) \left(\log f\left(\frac{g_f}{4}-1\right)\right)\defas \widehat{\epsilon} \\
  &\Rightarrow \left(H(X_i \mid X_A, X_j) +H(X_A, X_j)\right)-
	      \left( H(X_i \mid X_A, X_j, X_B) + H(X_A, X_j \mid X_B)\right) < \widehat{\epsilon} \\
  &\Rightarrow H(X_i \mid X_A, X_j) - H(X_i \mid X_A, X_j, X_B) < \widehat{\epsilon}, 
 \end{array}
\end{equation*}
where the first implication follows from marginalizing irrelevant variables and the second implication follows from (\ref{eqn_hleqp}). Using this we have that,
\begin{equation*}
 \begin{array}{rl}
  H(X_i\mid X_A, X_j, X_l) &\geq H(X_i\mid X_A, X_j, X_l, X_B)\\
  &= H(X_i\mid X_A, X_j, X_B)\;\; \mbox{since } X_i \displaystyle \overset{X_A, X_j, X_B}{\ci} X_l\\
  &> H(X_i\mid X_A, X_j) - \widehat{\epsilon}\\
 \end{array}
\end{equation*}
Using a similar argument, we also have,
\begin{equation*}
H(X_i\mid X_A, X_l, X_u) > H(X_i\mid X_A, X_l) - \widehat{\epsilon}
\end{equation*}
Combining the two inequalities, and using the fact that under the given conditions $\widehat{\epsilon} < \frac{\epsilon}{8}$, we get
\begin{equation*}
H(X_i\mid X_A, X_j)\leq H(X_i\mid X_A,X_u) - \frac{3\epsilon}{4}.
\end{equation*}
\end{proof}

\begin{lemma}\label{lemma_greedyalgo_requirement_empirical}
Consider a graphical model $G$ in which each node takes values in $\mathcal{X}$.
Let the number of samples be
\begin{align*}
n> 2^{15}\epsilon^{-4} |\mathcal{X}|^{4(D+2)}\left( (D+2)\log 2|\mathcal{X}|+2 \log \frac{p}{\delta}\right)
\end{align*}
Then $\forall \; i\in G$,
with probability greater than $1-\frac{\delta}{p}$, we have that
$\forall \; A\subseteq N(i), \; u \notin N(i)$
\begin{equation*}
 \left|H(X_i \mid X_A,X_u) - \widehat{H}(X_i \mid X_A,X_u)\right| < \frac{\epsilon}{8}
\end{equation*}

\end{lemma}
\begin{proof}
We use the fact that given sufficient samples, the empirical measure is close to the true measure uniformly in probability. Specifically, given any subset $A\subseteq V$ of nodes and any fixed $x_A\in\mathcal{X}^{|A|}$, we have by Azuma's inequality after $n$ samples,
\begin{equation*}
 \bP\left[ \left|P(x_A) - \widehat{P}(x_A)\right| > \gamma \right] < 2\exp(-2 \gamma^2 n) < \frac{2\delta}{p^2 (2|\mathcal{X}|)^{(D+2)}}.
\end{equation*}
where $\gamma = 2^{-8}\epsilon^{2} |\mathcal{X}|^{-2(D+2)}$. Let $V$ be the set of all vertices.
Now, by union bound over every $A \subseteq N(i), \; u \in V$ and $x_i,x_A,x_u$, we have
\begin{equation*}
 \bP\left[ \left|P(x_i,x_A,x_u) - \widehat{P}(x_i,x_A,x_u)\right| > \gamma \right] < \frac{\delta}{p}.
\end{equation*}
\eqref{eqn_hleqp} then implies
\begin{equation*}
 \bP\left[ \left|H(X_i \mid X_A,X_u) - \widehat{H}(X_i \mid X_A,X_u)\right| > \frac{\epsilon}{8} \right] < \frac{\delta}{p}.
\end{equation*}
giving us the required result.
\end{proof}
Using Lemmas \ref{lemma_greedyalgo_requirement} and \ref{lemma_greedyalgo_requirement_empirical}, we have the following :
$\forall \; i\in G$,
such that Assumptions \ref{assump_nondegeneracy} and \ref{assump_corrdecay} are satisfied,
with probability greater than $1-\frac{\delta}{p}$, we have that
$\forall \; A\subseteq N(i), \; u \notin N(i), \;\exists \; j \in N(i)\setminus A$ such that
\begin{equation}\label{eqn_greedyalgo_empiricalcondn1}
 \widehat{H}(X_i \mid X_A,X_j) < \widehat{H}(X_i \mid X_A,X_u) - \frac{\epsilon}{2}
\end{equation}
\noindent and $\forall \; i \in G$, such that Assumptions \ref{assump_nondegeneracy} and \ref{assump_corrdecay} are satisfied,
$\forall A\subset N(i), \; j \in N(i)\setminus A$, we have that
\begin{equation}\label{eqn_greedyalgo_empiricalcondn2}
\widehat{H}(X_i\mid X_A, X_j) < \widehat{H}(X_i\mid X_A) - \frac{\epsilon}{2}
\end{equation}

\begin{proof}[Theorem \ref{thm_main}]
The proof is based on mathematical induction.
The induction claim is as follows: just before entering an iteration of the WHILE loop, $\widehat{N}(i) \subset N(i)$.
Clearly this is true at the start of the WHILE loop since $\widehat{N}(i) = \Phi$. Suppose it is true just after entering the $k^{\mbox{th}}$ iteration.
If $\widehat{N}(i)=N(i)$ then clearly $\forall j \in V \setminus \widehat{N}(i), \;H(X_i \mid X_{\widehat{N}(i)}, X_j) = H( X_i \mid X_{\widehat{N}(i)})$.
Since with probability greater than $1-\frac{\delta}{p}$ we have that $\left| \widehat{H}(X_i \mid X_{\widehat{N}(i)}, X_j) - H(X_i \mid X_{\widehat{N}(i)}, X_j)\right| < \frac{\epsilon}{8}$
and $\left| \widehat{H}(X_i \mid X_{\widehat{N}(i)}) - H(X_i \mid X_{\widehat{N}(i)})\right| < \frac{\epsilon}{8}$, we also have
that $\left| \widehat{H}(X_i \mid X_{\widehat{N}(i)}, X_j) - \widehat{H}(X_i \mid X_{\widehat{N}(i)})\right| < \frac{\epsilon}{4}$.
So control exits the loop without changing $\widehat{N}(i)$. On the other hand, if $\exists j \in N(i) \setminus \widehat{N}(i)$
then from (\ref{eqn_greedyalgo_empiricalcondn2}) we have that
$\widehat{H}(X_i \mid X_{\widehat{N}(i)}) - \widehat{H}(X_i \mid X_{\widehat{N}(i)}, X_j) > \frac{\epsilon}{2}$.
So, a node is chosen to be added to $\widehat{N}(i)$ and control does not exit the loop. Now suppose for
contradiction that a node $u \notin N(i)$ is added to $\widehat{N}(i)$. Then we have that
$ \widehat{H}(X_i \mid X_{\widehat{N}(i)},X_u) < \widehat{H}(X_i \mid X_{\widehat{N}(i)},X_j)$. But this contradicts (\ref{eqn_greedyalgo_empiricalcondn1}). Thus, a neighbor
$ j \in N(i) \setminus \widehat{N}(i)$ is picked in the iteration to be added to $\widehat{N}(i)$, proving that
the neighborhood of $i$ is recovered exactly with probability greater than $1-\frac{\delta}{p}$. Using union bound, it is easy to
see that the neighborhood of each node (i.e., the graph structure) is recovered exactly with probability greater than $1-\delta$.
\end{proof}

\begin{remark}
 The proof for Theorem \ref{thm_main} can also be used to provide node-wise guarantees, i.e., for every node satisfying Assumptions \ref{assump_nondegeneracy} and \ref{assump_corrdecay}, if the number of samples is sufficiently large (in terms of its degree, and the length of the smallest cycle it is part of), its neighborhood will be recovered exactly with high probability.
\end{remark}

\begin{remark}
Any decreasing correlation-decay function $f$ suffices for Theorem \ref{thm_main}. However, the faster the correlation decay, the smaller the girth in the sufficient condition for Theorem \ref{thm_main} needs to be.
\end{remark}

And finally we have a corollary for the computational complexity of GreedyAlgorithm$(\epsilon)$ when executed on a graphical model that satisfies the conditions required by Theorem  \ref{thm_main}.

\begin{coro}
The expected run time of Algorithm \ref{GA} is  $O \left( \delta n p^3 + (1-\delta) Dnp^2 \right)$. Further, if $\delta$ is chosen to be $O(p^{-1})$, the sample complexity $n$ is $O(\log p)$ and the expected run time of Algorithm \ref{GA} is $O(Dp^2 \log p)$.
\end{coro}

\begin{proof}
For the second part, note that with probability greater than $1-\delta$, the algorithm recovers the correct graph
structure exactly. In this case, the number of iterations of the $while$ loop is bounded by $D$ for each node
. The time taken to compute any conditional entropy is bounded by $O(n)$. Hence the total run time is $O(Dnp^2)$.
Using the previous worst case bound on the running time, we obtain the result.
\end{proof}

\section{Guarantees for the Recovery of an Ising Graphical Model}\label{section_isingguarantees}

To aid in the interpretation of our results and comparison to the performance of other algorithms, we now specialize Theorem \ref{thm_main} to derive a self-contained result (i.e. we do not need to make additional use of Assumptions \ref{assump_nondegeneracy} and \ref{assump_corrdecay}) for the case of the widely-studied (zero-field) Ising graphical model \citep{isingmodel}. We define it below for completeness:
\begin{defn}
A set of random variables $\{X_v\mid v \in V\}$ are said to be distributed according to a zero field Ising model if
\begin{enumerate}
 \item $X_v \in \{-1,1\} \; \forall v \in V$ and \\
 \item $P(x_V) = \frac{1}{Z} \displaystyle \prod_{i,j \in V} \exp(\theta_{ij}x_i x_j)$
\end{enumerate}
where $Z$ is a normalizing constant. The Markov graph of such a set of random variables is given by
$G(V,E)$ where $E = \left\{ \{i,j\} \mid \theta_{ij} \neq 0 \right\}$.
\end{defn}

The following is a corollary of Theorem \ref{thm_main}. We note that while it may be possible to derive a stronger guarantee for Ising models (this is also suggested by experiments), we focus on just applying Theorem \ref{thm_main} as is, and obtaining a set of transparent and self-contained conditions in terms of natural parameters of the model.

\begin{theorem}\label{thm_ising}
Consider a zero-field Ising model on a graph $G$ with maximum degree $D$. Let the edge parameters $\theta_{ij}$ be
bounded in the absolute value by $0 < \beta < |\theta_{ij}| < \frac{\log 2}{2D}$. Let
$\epsilon \defas 2^{-10} \sinh^2 (2\beta)$.
If the girth of the graph satisfies $g>
 \frac{2^{15}}{\log 2}\left\{ D^2\log 2 - \log\left( \sinh 2\beta \right) \right\}$
then with samples $n = \xi \epsilon^{-4}\log \frac{p}{\delta}$ (where $\xi$ is a constant independent of $\epsilon, \delta, p$), GreedyAlgorithm$(\epsilon)$ outputs the exact structure of $G$ with probability greater than $1-\delta$.
\end{theorem}

The proof of this theorem consists of showing that  such an Ising  model satisfies Assumptions \ref{assump_nondegeneracy} and \ref{assump_corrdecay}, and the other conditions of Theorem \ref{thm_main}.
In Section \ref{subsec_ising_corrdecay}, we show that under certain conditions, an Ising model has an almost exponential correlation decay. Then in Section \ref{subsec_ising_nondegen}, we use the correlation decay of Ising models to show that under some further conditions, they also satisfy Assumption \ref{assump_nondegeneracy} for non-degeneracy. Combining the two, we get the above sufficient conditions for GreedyAlgorithm$(\epsilon)$ to learn the structure of an Ising graphical model with high probability.

\subsection{Correlation Decay in Ising Models}
\label{subsec_ising_corrdecay}

We will start by proving the validity of Assumption \ref{assump_corrdecay} in the form of the following proposition.
\begin{prop} 
\label{prop_isingmodel_setcorrelationdecay}
Consider a zero-field Ising model on a graph $G$ with maximum degree $D$ and girth $g$. Let the edge parameters $\theta_{ij}$ be
bounded in the absolute value by $|\theta_{ij}| < \frac{\log 2}{2D}$. Then, for any node $i$, its neighbors $N^1(i)$, its $2$-hop neighbors $N^2(i)$ and a set of nodes $A$, we have
\begin{eqnarray*}
\left| P(x_i,x_{N^1(i)},x_{N^2(i)} \mid x_A) - P(x_i,x_{N^1(i)},x_{N^2(i)}) \right| <c \exp\left( -\frac{\log 2}{3} \min\left( d(i,A), \frac{g}{2}-1 \right) \right)
\end{eqnarray*} 

$\forall \; x_i, x_{N^1(i)},x_{N^2(i)}$ and $x_A$ (where $c$ is a constant independent of $i$ and $A$).
\end{prop}


The outline of the proof of Proposition \ref{prop_isingmodel_setcorrelationdecay} is as follows. First, we show that if a subset of nodes is conditioned on a Markov blanket (i.e., on another subset of nodes which separates them from the remaining graph), then their potentials remain the same. For this we have the following lemma.
\begin{lemma} \label{lemma_separation_independence}
Consider a graphical model $G(V,E)$ and the corresponding factorizable probability distribution function $P$.
Let $A,B$ and $C$ be a partition of $V$ and $B$ separate $A$ and $C$ in $G$. Let $\tilde{G}(A\cup B,\tilde{E})$
be the induced subgraph of $G$ on $A\cup B$, with the same edge potentials as $G$ on all its edges and $\tilde{P}$ be
the corresponding probability distribution function. Then, we have that $P(x_D \mid x_B) = \tilde{P}(x_D \mid x_B) \;
\forall \; x_D, x_B$ where $D \subseteq A$.
\end{lemma}

Now, for any node $i$, the induced subgraph on all nodes which are at distance less than $\frac{g}{2}-1$ is a tree. Thus we can concentrate on proving correlation decay for a tree Ising model. We do this through the following steps:
\begin{enumerate}
\item Without loss of generality, the tree Ising model can be assumed to have all positive edge parameters
\item The worst case configuration for the conditional probability of the root node is when all the
leaf nodes are set to the same value and all the edge parameters are set to the maximum possible value
\item For this scenario, correlation decays exponentially
\end{enumerate}
The following three lemmas encode these three steps. For proofs, refer the Appendix.

\begin{lemma}\label{lemma_postoneg_theta}
Consider a tree Ising graphical model $T$. Let the corresponding probability distribution be $P$. Replace all
the edge parameters on this graphical model by their absolute values. Let the corresponding probability distribution
after this change be $\tilde{P}$. Then, there exists a set of bijections \\
$\left\{ M_v:\{-1,1\}\rightarrow \{-1,1\} \mid v \in V \setminus \{r\} \right\}$ where $V$ is the set of vertices
and $r$ is the root node such that,
$\forall x_r,x_{V\setminus r}$ we have that $P(x_r,x_{V \setminus r}) = \tilde{P}(x_r, M_v( x_v), v \in V\setminus r)$.
\end{lemma}
\begin{lemma}\label{lemma_worstcase_config}
 For a tree Ising graphical model $T$ with root $r$ and set of leaves $L$, we have
\begin{equation*}
(x_r = 1,x_L = 1) \in \displaystyle \arg\max_{x_r,x_L} \left| P(x_r \mid x_L ) - P(x_r) \right| 
\end{equation*}
\end{lemma}
\noindent And finally we have the following lemma.
\begin{lemma} 
\label{lemma_isingtree_setcorrelationdecay}
In a tree Ising model, suppose $\left|\theta_{ij}\right|<\gamma< \frac{\log 2}{2D}$ where $D$ is the maximum degree
of the graph. Then we have exponential correlation decay between the root node $r$, its neighbors $N^1(r)$, its $2$-hop
neighbors $N^2(r)$ and the set of leaves $L$ i.e.,
\begin{align*}
 \left| P(x_r,x_{N^1(r)},x_{N^2(r)} \mid x_L) - P(x_r,x_{N^1(r)},x_{N^2(r)}) \right| < c \exp( -\frac{\log 2}{3} d(r,L) )
\end{align*}
where $c$ is a constant independent of the nodes considered.
\end{lemma}

%
%

\subsection{Non-degeneracy in Ising Models with Correlation Decay}
\label{subsec_ising_nondegen}

Now using the results from the previous section, we turn our attention to the question of non-degeneracy. In particular, we have the following lemma which says that if an Ising graphical model has almost exponential correlation decay and its edge parameters satisfy certain conditions, then it also satisfies Assumption \ref{assump_nondegeneracy}.
For the proof, refer the Appendix.
\begin{lemma}\label{lemma_ising_nondegen}
Consider an Ising graphical model with edge parameters $\theta_{ij}$ bounded in the absolute value by
$0 < \beta < |\theta_{ij}| < \gamma$, max degree $D$, and having correlation decay as follows
\begin{equation*}
   \left| P(x_i,x_{N^1(i)},x_{N^2(i)}) - P(x_i,x_{N^1(i)},x_{N^2(i)} | x_B) \right| 
   < c \exp\left(-\alpha \min\left(d(i,B),\frac{g-2}{2}\right)\right)
\end{equation*}

\noindent $\forall \; i, B, x_i, x_{N^1(i)},x_{N^2(i)}$. If the girth 
$g> 2+\frac{2}{\alpha}\Big\{ (2D+11)\log 2 + \log c + \log\left( 1+2^D e^{2\gamma} \right) + 2 \gamma (D+3)-\log\left| \sinh 2\beta \right| \Big\}$,
then this graphical model satisfies Assumption \ref{assump_nondegeneracy}
with $\epsilon = 2^{-7}e^{-6\gamma D} \sinh^2 (2\beta)$.
\end{lemma}

Finally, the proof of Theorem \ref{thm_ising} follows directly by combining Theorem \ref{thm_main}, Proposition \ref{prop_isingmodel_setcorrelationdecay} and Lemma \ref{lemma_ising_nondegen}. For complete details, refer the Appendix.

\section{Simulations}\label{section_simulations}
In this section, we present the results of numerical experiments evaluating the performance of our
algorithm. We note that to satisfy the conditions so that our theoretical
guarantees are applicable, the graph should have a large girth. However, it seems that the strong sufficiency
conditions are a result of our analysis. In fact our algorithm seems to work well even on graphs with
small girth. To demonstrate this fact we perform our experiments on graphs with small girth.
$\epsilon$, which is an input to the algorithm is chosen empirically.

In the first experiment, we evaluate our algorithm on grids of various sizes. Fig. \ref{fig:grid_comparison} compares
the sample complexity and computational complexity of our algorithm to those of \citep{ravikumar} which will be henceforth
referred to as RWL. Note that RWL is specifically tailored to the Ising model, and leverages this to yield lower
sample complexity. Ours is a generic algorithm that can be used for any discrete graphical model, and thus requires more
(but comparable) number of samples. It can be seen however that our algorithm is much faster than RWL.

\begin{figure}[ht]
  \centering
  \subfigure[]
  {
    \includegraphics[width=8.6cm,scale=1]{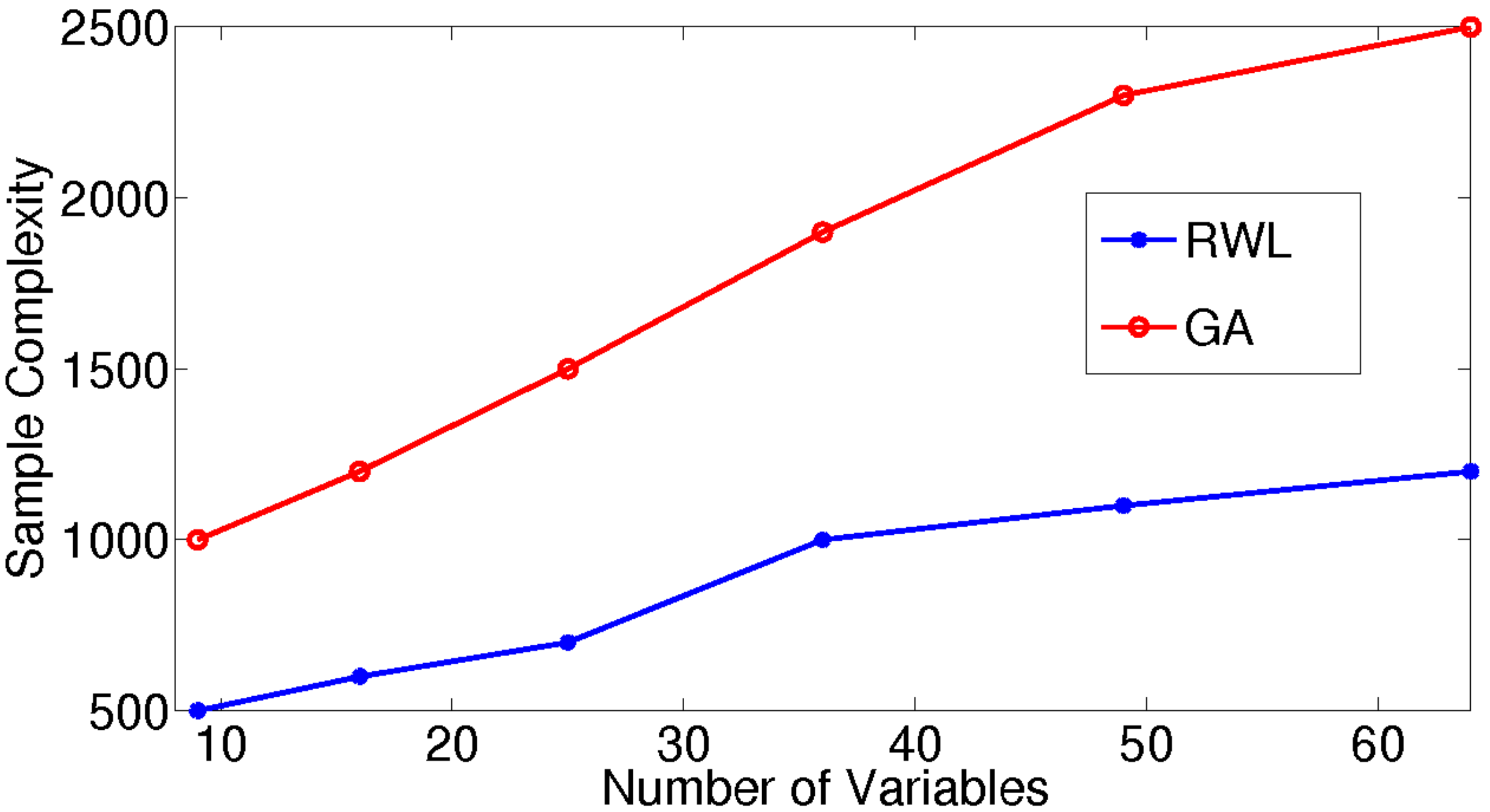}
    \label{fig:grid_samplecomplexity}
  }
  \subfigure[]
  {
    \includegraphics[width=8.8cm,scale=1]{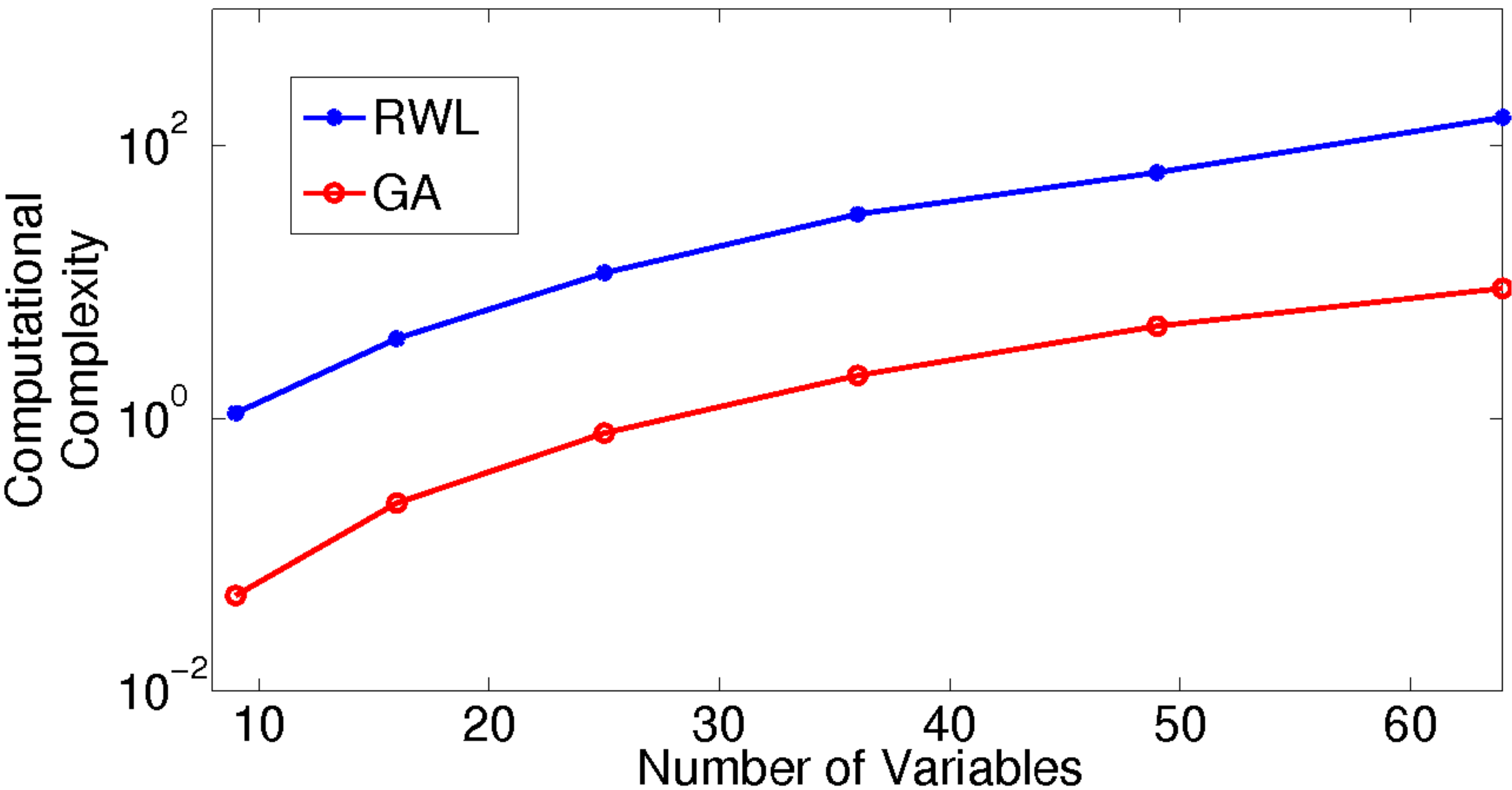}
    \label{fig:grid_compcomplexity}
  }
  \vspace{-.3cm}
  \caption{
    Plots of
    (a) sample complexity and
    (b) computational complexity
    of our algorithm (GA) and that of \citep{ravikumar} (RWL) for various grid sizes. Edge parameters are all chosen to be equal to $0.5$.
    X-axis represents the number of variables ($9$ for a $3\times 3$ grid,
    $16$ for a $4\times 4$ grid and so on). In (a), Y-axis represents the sample complexity
    which is taken to be the minimum number of samples required to obtain a probability of
    success of $0.95$. In (b), Y-axis is in logarithmic scale and represents the time taken
    in seconds for a single run using the number of samples from (a). All the above quantities
    are calculated by averaging over $50$ runs.
  }
  \label{fig:grid_comparison}
\end{figure}

Finally, we present an application of our algorithm to model senator interaction graph using the senate voting records,
following \citep{banerjee}. A \emph{Yea} vote is treated as a $1$ where as a \emph{Nay} vote or \emph{absentee} vote
is treated as $-1$. To avoid bias, we only consider senators who have voted in a fraction of atleast $0.75$ of all the bills
during the years 2009 and 2010. The output graph is presented in Fig. \ref{fig:senategraph}.

\begin{figure}[ht]
  \centering
  \includegraphics[width=1\linewidth]{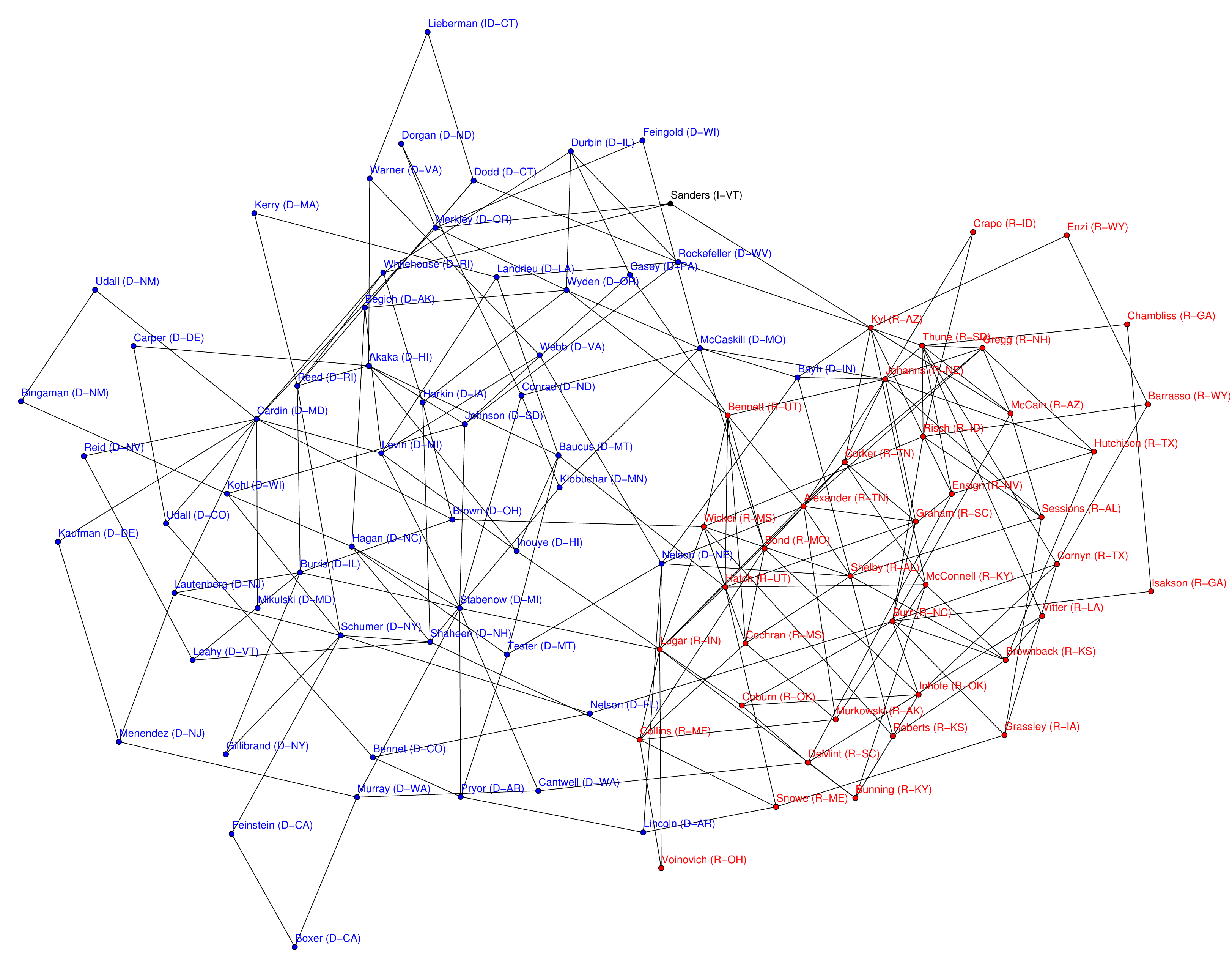}
  \caption{Blue nodes represent democrats, red nodes represent republicans and black node represents an independent.
	    We can make some preliminary observations from the graph. Most of the democrats are connected to other
	    democrats and most of the republicans are connected to other republicans (in particular, the number of
	    edges between democrats and republicans is approximately $0.1$ fraction of the total number of edges).
	    The senate minority leader, McConnell is well connected to other republicans where as the senate majority
	    leader, Reid is not well connected to other democrats. Sanders and Lieberman, both of who caucus with
	    democrats have more edges to democrats than to republicans. We use the graph drawing algorithm of Kamada
	    and Kawai to render the graph \citep{kamada}.
	  }
  \label{fig:senategraph}
\end{figure}

\section{Discussion}\label{section_discussion}

We developed a simple greedy algorithm for Markov structure learning. The algorithm is simple to implement and has low computational complexity. We then showed that under some non-degeneracy, correlation decay, maximum degree and girth assumptions on the MRF, our algorithm recovers the correct graph structure with $O(\epsilon^{-4}\log \frac{p}{\delta})$ samples. We then specialize our conditions to prove a self-contained result for the most popular discrete graphical model - the Ising model.

The success of our algorithm can be further improved by post-processing via {\em pruning}. In particular, as mentioned, the neighborhood of a node as estimated by our algorithm always includes the true neighborhood -- but it may also include spurious nodes. The latter can be then identified by checking each node of the estimated neighborhood, to see if it actually provides a reduction in conditional entropy over and above all the other nodes. Analysis of the improvement achieved by such a procedure is more challenging, but it may be likely that doing so will reveal an algorithm that can handle much larger degrees and smaller girths.

\acks{This work was partially supported by ARO grant W911NF-10-1-0360.
We thank Jason K. Johnson for letting us use his graph drawing code
for the senator graph.
}


\appendix \label{appendix}
\section*{Appendix}
We will first prove the lemmas required for proving Proposition \ref{prop_isingmodel_setcorrelationdecay}
\begin{proof}[Lemma \ref{lemma_postoneg_theta}]
The proof is by construction. For each node $v \in V$, let $M_v(x_v) = \eta_v x_v$.
For the root node, let $\eta_r \defas 1$. For any other node $v$, let $u$ be the parent of $v$ in the rooted
tree with root $r$. Define $\eta_v \defas \frac{\theta_{uv}}{\left| \theta_{uv} \right|}\eta_u$. Let $\Phi$
and $\widetilde{\Phi}$ be the potential functions corresponding to $P$ and $\widetilde{P}$ respectively. Then,
\begin{equation*}
 \begin{array}{rl}
  \Phi(x_V) &= \displaystyle \prod_{uv \in T} \exp\left(\theta_{uv}x_u x_v \right) \\
	    &= \displaystyle \prod_{uv \in T} \exp\left(\left| \theta_{uv} \right| \frac{\theta_{uv}}{\left| \theta_{uv} \right|} \eta_u^2 x_u x_v \right) \\
	    &= \displaystyle \prod_{uv \in T} \exp\left(\left| \theta_{uv} \right| \eta_u \eta_v x_u x_v \right) \\
	    &= \displaystyle \prod_{uv \in T} \exp\left(\left| \theta_{uv} \right| M_u(x_u) M_v(x_v) \right) \\
	    &= \widetilde{\Phi}(x_r,M_v(x_v), v \in V\setminus r)\\
 \end{array}
\end{equation*}

\noindent Since the potential functions are preserved by the bijections, so are the probabilities.
\end{proof}

We will first prove the following lemma which will help us in proving Lemma \ref{lemma_worstcase_config}.
\begin{lemma}\label{lemma_allpostheta_monotonicity}
Consider a tree Ising graphical model $T$ with root $r$, set of leaves $L$
and all positive edge parameters. Let $P$ be its probability distribution. Then, the quantity
$P(X_r = 1 \mid X_L=x_L)$ is monotonically increasing in $x_l,\; \forall \; l \in L$.
Moreover, $P(X_r = 1 \mid X_L=1)$ is monotonically increasing in $\theta_{ij} \; \forall \; \{i,j\} \in T$.
\end{lemma}
\begin{proof}
For simplicity of notation, we define $f(x_L) \defas P(X_r=1 \mid X_L = x_L)$.
Let us prove the above statement by induction on the depth of the tree. For a tree of depth $1$, we have that
\begin{small}
\begin{equation*}
 \begin{array}{rl}
  f(x_L) &= \frac{\displaystyle \prod_{l \in L}\exp(\theta_{rl}x_l)}{\displaystyle \prod_{l \in L}\exp(\theta_{rl}x_l) + \displaystyle \prod_{l \in L}\exp(-\theta_{rl}x_l)} \\
		&= \frac{\displaystyle \prod_{l \in L, l \neq \widetilde{l}}\exp(\theta_{rl}x_l)}{\displaystyle \prod_{l \in L,l\neq \widetilde{l}}\exp(\theta_{rl}x_l) + \exp(-2\theta_{r\widetilde{l}}x_{\widetilde{l}})\displaystyle \prod_{l \in L,l\neq \widetilde{l}}\exp(-\theta_{rl}x_l)} \\
 \end{array}
\end{equation*}
\end{small}
\noindent Since $\theta_{r\widetilde{l}}>0$, $f(x_L)$ increases when $x_{\widetilde{l}}$ is changed from $-1$ to $1$.

Now, suppose the statement is true for all trees of depth upto $k$. Consider a tree of depth $k+1$, with root $r$.
Let $N(r)$ be the set of children of $r$. For every $c \in N(r)$, let $T_c$ be the subtree rooted at $c$ with the same
edge parameters as in $T$ and $L_c$ be the leaves of $T_c$. Let $P_c$ be the probability measure corresponding to $T_c$ and
$f_c(x_{L_c}) \defas P_c(x_c = 1 \mid x_{L_c} )$. Then, the conditional probability
of the root node can be written as


\begin{align} \label{eqn_prob_recursive}
  f(x_L) = \frac{\displaystyle \prod_{c \in N(r)} \left( \exp(\theta_{rc}) f_c(x_{L_c}) +\exp(-\theta_{rc}) \left(1-f_c(x_{L_c})\right) \right)}
  {B}
\end{align}
\noindent where
\begin{align*}
  B = &\displaystyle \prod_{c \in N(r)} \left( \exp(\theta_{rc}) f_c(x_{L_c}) +\exp(-\theta_{rc}) \left(1-f_c(x_{L_c})\right) \right) + \\
 &\displaystyle \prod_{c \in N(r)} \left( \exp(-\theta_{rc}) f_c(x_{L_c}) +\exp(\theta_{rc}) \left(1-f_c(x_{L_c})\right) \right)
\end{align*}
\eqref{eqn_prob_recursive} can now be manipulated to obtain \eqref{eqn_prxl}.

\begin{equation}\label{eqn_prxl}
  f(x_L) 
    = \frac{K_1}{ 
		      K_1 +
		      K_2\frac{g_{\widetilde{c}}(x_{\widetilde{c}})+\exp({2\theta_{r\widetilde{c}}})}
				{g_{\widetilde{c}}(x_{\widetilde{c}})\exp({2\theta_{r\widetilde{c}}})+1}
		}
\end{equation}
\noindent where $g_{\widetilde{c}}(x_{\widetilde{c}}) = \frac{f_{\widetilde{c}}(x_{L_{\widetilde{c}}})}{1-f_{\widetilde{c}}(x_{L_{\widetilde{c}}})}$,
and $K_1$ and $K_2 > 0$ are independent of $x_{L_{\widetilde{c}}}$ and $\theta_{r\widetilde{c}}$.
Since $K_2>0$ and $\theta_{r\widetilde{c}}>0$, $f(x_L)$ increases if $f_{\widetilde{c}}(x_{L_{\widetilde{c}}})$ increases.
So, for any leaf node, if its value changes from $-1$ to $1$, the corresponding $f_{\widetilde{c}}(x_{L_{\widetilde{c}}})$
increases and hence $f(x_L)$ increases, proving the induction claim.

Using the same induction argument as above and noting that $f(x_L=1) > \frac{1}{2}$, it can be seen that
$f(x_L=1)$ is monotonically increasing in $\theta_{ij}\; \forall \{i,j\} \in T$.
\end{proof}

\begin{proof}[Lemma \ref{lemma_worstcase_config}]
We know that $P(x_r) = \frac{1}{2}$ for $x_r = \pm 1$. Clearly any $x_L$ that maximizes $\left| P(x_r \mid x_L ) - P(x_r) \right|$
should either minimize or maximize $P(x_r \mid x_L )$. Note also that there is a one-one correspondence between such
configurations (i.e., for every maximizing configuration, there exists a minimizing configuration such that
both of them maximize $\left| P(x_r \mid x_L ) - P(x_r) \right|$).
From Lemma \ref{lemma_allpostheta_monotonicity}, we know that $x_L=1$ maximizes $P(x_r =1 \mid x_L )$ and by symmetry this should be the same as
$P(x_r = -1\mid x_L = -1)$ and equal $\displaystyle \max_{x_L} P(x_r = -1 \mid x_L )$.
So, we can conclude that $\left| P(x_r \mid x_L ) - P(x_r) \right|$ is maximized by $(x_r=1,x_L=1)$.
\end{proof}
%
%
%
\begin{lemma}\label{lemma_isingtree_correlationdecay}
Consider a tree Ising model $T$ with root node $r$, set of leaves $L$ and maximum degree $D$.
Let $P$ be its probability measure.  Suppose the absolute values of the edge parameters are bounded by
$\left| \theta_{ij} \right| < \frac{\log 2}{2D} \; \forall \; \{i,j\} \in T$. Then, we have that
$\left| P(x_r \mid x_L) - P(x_r) \right| < \exp( -\frac{\log 2}{3} d(r,L) ) \; \forall x_r, x_L$.
\end{lemma}
\begin{proof}
Using Lemmas \ref{lemma_postoneg_theta}, \ref{lemma_worstcase_config} and \ref{lemma_allpostheta_monotonicity},
we can assume without loss of generality that the parameters $\theta_{ij}$ on all the edges are positive and
equal to $\frac{\log 2}{2D}$ (which is the maximum possible value), consider a complete D-ary tree and
concentrate on $\left| P(X_r=1 \mid X_L=1) - P(X_r=1) \right|$. For simplicity of notation, let
$\theta \defas \frac{\log 2}{2D}$. For a tree of depth $d$, let $a(d) \defas P(X_r =1\mid X_L = 1)$. We have that
\begin{equation*}
 a(d+1)
 = \frac{\left( \exp(\theta) a(d) + \exp(-\theta) \left(1-a(d)\right)\right)^D}
	  { 
	    \left( \exp(\theta) a(d) + \exp(-\theta) \left(1-a(d)\right)\right)^D +
	    \left( \exp(-\theta) a(d) + \exp(\theta) \left(1-a(d)\right)\right)^D
	  }
\end{equation*}
\noindent Using some algebraic manipulations and substituting the value of $\theta$, we obtain
\begin{equation*}
 \left| a(d+1) - \frac{1}{2} \right| < \exp\left(-\frac{\log 2}{3}\right) \left| a(d) - \frac{1}{2} \right|
\end{equation*}
and the result follows.
\end{proof}
We need the following lemma to prove Lemma \ref{lemma_isingtree_setcorrelationdecay}.
\begin{lemma}
\label{lemma_isingtree_neighborcorrelationdecay}
Consider a tree Ising model $T$, with root node $r$, set of leaves $L$ and maximum degree $D$.
Let $P$ be its probability measure.  Suppose the absolute values of the edge parameters are bounded by
$\left| \theta_{ij} \right| < \frac{\log 2}{2D} \; \forall \; \{i,j\} \in T$. Then, $\forall c$ such that
$c$ is a child of $r$, we have that
$\left| P(x_c \mid x_r,x_L) - P(x_c \mid x_r) \right| < 4 \exp( -\frac{\log 2}{3} d(r,L) ) \; \forall x_r, x_j, x_L$.
\end{lemma}

\begin{proof}
Using Lemma \ref{lemma_postoneg_theta}
we can assume without loss of generality that the parameters $\theta_{ij}$ on all the edges are positive.
$(x_c,x_r)$ can take values $(\pm 1, \pm 1)$. For each of those values, the value of $x_L$ that
maximizes $\left| P(x_c \mid x_r, x_L) - \right.$\\
$\left.P(x_c \mid x_r) \right|$ either maximizes or minimizes
$P(x_c \mid x_r,x_L)$. Noting from (a slight extension to) Lemma \ref{lemma_allpostheta_monotonicity}
that $P(x_c \mid x_r,x_L)$ is monotonic in $x_L$, it suffices to consider the eight possibilities
$\left| P(X_c=\pm 1 \mid X_r=\pm 1, X_L=\pm 1) - \right.$\\
$\left.P(X_c=\pm 1 \mid X_r=\pm 1) \right|$.
We show how to calculate the above value for $x_c=1,x_r=1,x_L=1$. Interested readers can check that
the conclusions below apply to all the other cases as well.
Using Lemma \ref{lemma_allpostheta_monotonicity}, we can assume that the parameters 
$\theta_{ij}$ on all the edges except the edge $\{r,c\}$ are equal to $\frac{\log 2}{2D}$ and consider
a complete D-ary tree. Let $\theta \defas \theta_{rc}$.
We know that $P(X_c=1 \mid X_r=1) = \frac{\exp(\theta)}{\exp(\theta) + \exp(-\theta)}$.
Let $d$ be the depth of the tree and $b(d) \defas P(X_c =1 \mid X_r=1,X_L = 1)$.
We have $b(d) = \frac{\exp(\theta) a(d-1)}{\exp(\theta) a(d-1) + \exp(-\theta) \left(1-a(d-1)\right)}$
where $a(d)$ is as defined in Lemma \ref{lemma_isingtree_correlationdecay}.
%
Using some algebraic manipulations, it can be shown that $\left| b(d) - \frac{\exp(\theta)}{\exp(\theta) + \exp(-\theta)} \right|
< 2 \left|a(d-1) - \frac{1}{2}\right|$. Using Lemma \ref{lemma_isingtree_correlationdecay} finishes the proof.
\end{proof}
\begin{proof}[Lemma \ref{lemma_isingtree_setcorrelationdecay}]
Using Lemma \ref{lemma_isingtree_neighborcorrelationdecay}, we have
 \begin{equation*}
  \begin{array}{rl}
   &\left| P(x_r,x_{N^1(r)},x_{N^2(r)} \mid x_L) - P(x_r,x_{N^1(r)},x_{N^2(r)}) \right| \\
   &= \left| P(x_r \mid x_L) \displaystyle \prod_{j \in N^1(r)} P(x_j \mid x_r,x_L) 
			     \displaystyle \prod_{k \in N^2(r)} P(x_k \mid x_j,x_L) \right. \\
   &\;\;- \left.P(x_r) \displaystyle \prod_{j \in N^1(r)} P(x_j \mid x_r)
			     \displaystyle \prod_{k \in N^2(r)} P(x_k \mid x_j) \right| \\
   &< 2^{D^2+3} \exp\left( -\frac{\log 2}{3} \left(d(r,L)-1\right) \right)\\
   &= c \exp\left( -\frac{\log 2}{3} d(r,L) \right)
  \end{array}
 \end{equation*}
proving the result.
\end{proof}
\begin{proof}[Proposition \ref{prop_isingmodel_setcorrelationdecay}]
Let $I \defas \{i\}\cup N^1(i) \cup N^2(i)$. Let $B$ be a set that separates $I$ and $A$ such that
$d(I,B) = \min ( d(i,A), \frac{g}{2}-1 )$. Let $J$ be the component of nodes containing $I$ when the graph is separated
by $B$. We know that the induced subgraph on $J \cup B$ is a tree. Applying Lemma \ref{lemma_isingtree_setcorrelationdecay}
on this tree and using Lemma \ref{lemma_separation_independence}, we obtain $\left| P(x_I \mid x_B) - P(x_I \mid \widetilde{x}_B ) \right| <
2c \exp( -\frac{\log 2}{3} d(I,B) ) \; \forall x_I,x_B, \widetilde{x}_B$. Since $P(x_I)$ is a weighted average of
$P(x_I \mid x_B)$ for various $x_B$, we have 

\begin{equation*}
\left| P(x_I \mid x_B) - P(x_I) \right| < 2c \exp( -\frac{\log 2}{3} d(I,B) ) \; \forall x_I,x_B 
\end{equation*}

The result then follows since $P(x_I \mid x_A)$ is a weighted average of $P(x_I \mid x_B)$.
\end{proof}
\begin{proof}[Lemma \ref{lemma_ising_nondegen}]
Let the graphical model be denoted by $G(V,E)$, $\Phi(x_i,x_j) \defas \exp(\theta_{ij}x_ix_j)$ denote the potential
on edge $\{i,j\}$ when $X_i=x_i$ and $X_j=x_j$ and $\Phi(x_A)$ denote the potential due to all edges with both
vertices in $A$ when $X_A =x_A, \; \forall A \subseteq V$. In the following, we assume that the girth of the graph is 
$g > 4$. Consider a node $i$ and a subset of its neighbors
$j_1,\cdots, j_k, z$ and a node $w$ which is a neighbor of $z$. We know that the pairwise potentials satisfy
$\exp(-\gamma)<\Phi(x_i,x_j)<\exp(\gamma)$. Let $\breve{E} \defas E \setminus \left\{ \{i,j_1\},\cdots,\{i,j_k\},\{i,z\},\{z,w\} \right\}$
and consider the graph $\breve{G}(V,\breve{E})$ with the same potentials on all edges as in $G$.
Let $A \defas \{i,j_1,\cdots,j_k,z,w\}$ and choose any other set $B \subset V$. Let $P$ and $\breve{P}$ be the probability mass functions
corresponding to $G$ and $\breve{G}$ respectively. Similarly let $d(i,j)$ and $\breve{d}(i,j)$ be the distance between $i$ and $j$
in $G$ and $\breve{G}$ respectively. Suppose further that $d(i,B)=d$. Then, $\breve{d}(i,B)>d(A,B)=d$. Note that,

\begin{equation}\label{eqn_modified_actual_probrelation}
 \breve{P}(x_A,x_B) = \frac{1}{\breve{Z}} \frac{P(x_A,x_B)}{\Phi(x_A)}
\end{equation}
\noindent where $\breve{Z}$ is an appropriate normalizing constant. Note that
$\frac{1}{\breve{Z}} \displaystyle \sum_{x_A} \frac{P(x_A)}{\Phi(x_A)} = \displaystyle \sum_{x_A}\breve{P}(x_A) = 1$.
It follows from this that $\exp(-\gamma) < \frac{1}{\breve{Z}} < \exp(\gamma)$.
Using (\ref{eqn_modified_actual_probrelation}), the hypothesis that an Ising model has almost exponential correlation decay,
we obtain the following inequalities after some algebraic manipulations,
\begin{equation}\label{corrdecay_longeqn}
  |\breve{P}(x_A,x_B)-\breve{P}(x_A)\breve{P}(x_B)| <
  c 2^{D+3} \exp(4\gamma) \exp( -\alpha \min(d,\frac{g-2}{2}) ) P(x_B)
\end{equation}
%
\begin{equation}\label{marginal_eqn}
  \breve{P}(x_B) \geq
  \exp(-2\gamma) \left( 1-2^{D+2}c \exp( -\alpha \min(d,\frac{g-2}{2}) ) \right) P(x_B)
\end{equation}
$\forall x_A,x_B$.Combining (\ref{corrdecay_longeqn}) and (\ref{marginal_eqn}), we obtain
\begin{equation*}
 |\breve{P}(x_A,x_B)-\breve{P}(x_A)\breve{P}(x_B)| <
 c 2^{D+3}\exp(6\gamma) \frac{\exp(-\alpha \min(d,\frac{g-2}{2}))}{1-2^{D+2}c \exp( -\alpha \min(d,\frac{g-2}{2}) )} \breve{P}(x_B)
\end{equation*}
\noindent and subsequently by marginalizing, we obtain
\begin{equation*}
 |\breve{P}(x_i,x_B)-\breve{P}(x_i)\breve{P}(x_B)| <
 c 2^{2D+4}\exp(6\gamma) \frac{\exp(-\alpha \min(d,\frac{g-2}{2}))}{1-2^{D+2}c \exp( -\alpha \min(d,\frac{g-2}{2}) )} \breve{P}(x_B)
\end{equation*}
\noindent Let $A' \defas A \setminus \{i\}$. 
Since $d(i,A') = 2$, we have that $\breve{d}(i,A') \geq g-2$.
So, $\exists \; B \subseteq V$ separating $i$ and $A'$ in $\breve{G}$ such that $d(i,B)\geq \frac{g-2}{2}$. Then, $\forall \; x_i,x_{A'}$
\begin{equation}\label{eqn_modified_corrdecay}
 \begin{array}{rl}
  |\breve{P}(x_i \mid x_{A'}) - \breve{P}(x_i)|
  &= \left|\displaystyle \sum_{x_B}\left(\breve{P}(x_i \mid x_B)-\breve{P}(x_i)\right)\breve{P}(x_B \mid x_{A'})\right| \\
  &< c 2^{2D+4}\exp(6\gamma) \frac{\exp(-\alpha \frac{g-2}{2})}{1-2^{D+2}c \exp( -\alpha \frac{g-2}{2} )} \\
  &< 2^{-(D+6)} \exp(-2\gamma (D+1)) \left|\sinh(2\beta)\right| \defas \breve{\epsilon}
 \end{array}
\end{equation}
where the last inequality follows from the lower bound on girth $g$ in the hypothesis.

Now consider the graph $\widetilde{G}(V,\widetilde{E})$ where $\widetilde{E} \defas \left\{ \{i,j_1\},\cdots,\{i,j_k\},\{i,z\},\{z,w\} \right\}$.
Let the potentials on the edges in $\widetilde{G}$ be the same as those in $G$ and denote the corresponding probability
mass function by $\widetilde{P}$. Clearly, we have the following relation between $P, \breve{P}$ and $\widetilde{P}$.
\begin{equation*}\label{eqn_prob_w/wo_relation}
 P(x_A) = \frac{1}{Z}\breve{P}(x_A)\widetilde{P}(x_A) \; \forall \; x_A
\end{equation*}
where $Z$ is an appropriate normalizing constant.
Using (\ref{eqn_modified_corrdecay}) and the symmetry of the Ising model
(i.e., $\breve{P}(x_i) = \frac{1}{2}$ for $x_i = \pm 1$), we obtain
\begin{equation*}
 \begin{array}{rl}
  P(x_i\mid x_{A'}) &= \frac{P(x_i, x_{A'})}{P(x_{A'})} \\
		    &= \frac{\frac{1}{Z}\breve{P}(x_i, x_{A'})\widetilde{P}(x_i, x_{A'})}{\frac{1}{Z}\displaystyle \sum_{x_i}\breve{P}(x_i,x_{A'})\widetilde{P}(x_i,x_{A'})} \\
		    &= \frac{\breve{P}(x_i, x_{A'})\widetilde{P}(x_i, x_{A'})}{\displaystyle \sum_{x_i}\breve{P}(x_i \mid x_{A'})\breve{P}(x_{A'})\widetilde{P}(x_i,x_{A'})} \\
		    &< \frac{\breve{P}(x_i\mid x_{A'})\widetilde{P}(x_i\mid x_{A'})}{\left(\frac{1}{2}-\breve{\epsilon}\right)} \\
		    &< \frac{1+2\breve{\epsilon}}{1-2\breve{\epsilon}}\widetilde{P}(x_i\mid x_{A'})
 \end{array}
\end{equation*}
after some algebraic manipulations. Similarly, we also have
\begin{equation*}
P(x_i\mid x_{A'}) > \frac{1-2\breve{\epsilon}}{1+2\breve{\epsilon}}\;\widetilde{P}(x_i\mid x_{A'})
\end{equation*}
which implies
\begin{equation*}
\left|P(x_i\mid x_{A'}) - \widetilde{P}(x_i \mid x_{A'})\right| < 8\breve{\epsilon}
\end{equation*}
Finally, letting $A^* \defas A' \setminus \{z\}$, we have,
\begin{equation*}
 \begin{array}{rl}
  &H(X_i \mid X_{A^*}) - H(X_i \mid X_{A'})\\
  &= \displaystyle \sum_{x_{A'}} P(x_{A'}) \displaystyle \sum_{x_i} P(x_i \mid x_{A'}) \log\left(\frac{P(x_i \mid x_{A'})}{P(x_i \mid x_{A^*})}\right)\\
  &= \displaystyle \sum_{x_{A'}} P(x_{A'}) D\left(P(X_i \mid x_{A'}) || P(X_i \mid x_{A^*})\right)\\
  &\geq \frac{1}{2\log 2}\displaystyle \sum_{x_{A'}} P(x_{A'}) \displaystyle \sum_{x_i} |P(x_i \mid x_{A'}) - P(x_i \mid x_{A^*})|^2\\
  &= \frac{1}{2}\displaystyle \sum_{x_{A^*}} P(x_{A^*}) \displaystyle \sum_{x_z} P(x_z \mid x_{A^*})\displaystyle \sum_{x_i} |P(x_i \mid x_{A'}) - P(x_i \mid x_{A^*})|^2\\
  &\geq \frac{1}{2}\displaystyle \sum_{x_{A^*},x_i} P(x_{A^*})  \displaystyle \min_{x_z} P(x_z \mid x_{A^*})
  \frac{1}{2}|P(x_i \mid x_{A^*},x_z=-1) - P(x_i \mid x_{A^*}, x_z = 1)|^2\\
  &\geq \frac{1}{4}\displaystyle \sum_{x_{A^*},x_i} P(x_{A^*}) \frac{\exp(-\gamma D)}{\exp(\gamma D) + \exp(-\gamma D)} \\
  &\;\;\;\;\left( \max\left(0,\left|\widetilde{P}(x_i \mid x_{A^*},x_z=-1) - \widetilde{P}(x_i \mid x_{A^*},x_z=1)\right| - 16 \breve{\epsilon}\right)\right)^2\\
  &> \frac{1}{8}\displaystyle \sum_{x_{A^*},x_i} P(x_{A^*}) \exp(-2\gamma D) \left(\frac{\left|\sinh(2\beta)\right|\exp(-2\gamma D)}{2} - 16\breve{\epsilon}\right)^2 \\
  &> \frac{1}{128} \exp(-6\gamma D)\sinh^2(2\beta) \\
 \end{array}
\end{equation*}
So, we have shown that under the given conditions, an Ising model satisfies (\ref{eqn_nondegen_2}) with
$\epsilon=\frac{1}{128} \exp(-6\gamma D)\sinh^2(2\beta)$. It is straightforward to note that the above proof can
also be used to show that the Ising model also satisfies (\ref{eqn_nondegen_1}) with the same $\epsilon$, completing
the proof of the lemma.

\end{proof}
\begin{proof}[Theorem \ref{thm_ising}]
The theorem follows directly from Theorem \ref{thm_main}, Proposition \ref{prop_isingmodel_setcorrelationdecay}
and Lemma \ref{lemma_ising_nondegen}.
\end{proof}
\end{document}